\documentclass{article}

\usepackage[final]{bdl_2021_camera_ready}

\usepackage[utf8]{inputenc} %
\usepackage[T1]{fontenc}    %
\usepackage{hyperref}       %
\usepackage{url}            %
\usepackage{booktabs}       %
\usepackage{amsfonts}       %
\usepackage{nicefrac}       %
\usepackage{microtype}      %
\usepackage{xcolor}         %
\usepackage{graphicx}
\usepackage{amsmath}
\usepackage{amssymb}
\usepackage{amsfonts}
\usepackage{amsthm}
\usepackage{dsfont}
\usepackage{bbding}

\usepackage[font=footnotesize]{caption}
\usepackage{subcaption}

\usepackage{amsmath,amsfonts,bm}
\usepackage{algorithm,algorithmic}

\def\1{\bm{1}}

\def\epsilon{{\varepsilon}}

\def\vomega{{\bm{\omega}}}

\def\vx{{\bm{x}}}
\def\vy{{\bm{y}}}

\DeclareMathAlphabet{\mathsfit}{\encodingdefault}{\sfdefault}{m}{sl}
\SetMathAlphabet{\mathsfit}{bold}{\encodingdefault}{\sfdefault}{bx}{n}

\newcommand{\E}{\mathbb{E}}

\newcommand{\R}{\mathbb{R}}

\usepackage{amssymb}

\usepackage{xcolor}
\definecolor{mydarkblue}{rgb}{0,0.08,0.45}
\definecolor{mygreen}{rgb}{0.032, 0.6392, 0.2039}
\definecolor{mypurple}{HTML}{B266FF}

\usepackage[textwidth=2cm]{todonotes}

\def\xx{{\boldsymbol x}}

\def\xx{{\boldsymbol x}}

\newcommand\ignore[1]{}

\usepackage[inline, shortlabels]{enumitem}  %

\newtheorem{thminformal}{Theorem (informal)}
\newtheorem{theorem}{Theorem}

\usepackage{xcolor}         %
\definecolor{DarkGreen}{rgb}{0.1,0.5,0.1}
\definecolor{DarkRed}{rgb}{0.5,0.1,0.1}
\definecolor{DarkBlue}{rgb}{0.1,0.1,0.5}
\usepackage{hyperref}       %
\hypersetup{
    unicode=false,          %
    pdftoolbar=true,        %
    pdfmenubar=true,        %
    pdffitwindow=false,      %
    pdfnewwindow=true,      %
    colorlinks=true,       %
    linkcolor=DarkBlue,          %
    citecolor=DarkGreen,        %
    filecolor=DarkRed,      %
    urlcolor=DarkBlue,          %
    pdftitle={},
    pdfauthor={},
}
\let\svthefootnote\thefootnote
\newcommand\freefootnote[1]{%
  \let\thefootnote\relax%
  \footnotetext{#1}%
  \let\thefootnote\svthefootnote%
}

\title{The Peril of Popular Deep Learning \\ Uncertainty Estimation Methods}  %

\author{%
  Yehao Liu$^\star$ \\
  EPFL\\
  \And
  Matteo Pagliardini$^\star$\\
  EPFL\\
  \And
  Tatjana Chavdarova\\
  UC Berkeley\\
  \And
  Sebastian U. Stich\\
  CISPA\\
}

\begin{document}
\maketitle

\begin{abstract}
Uncertainty estimation (UE) techniques---such as the Gaussian process (GP), Bayesian neural networks (BNN), Monte Carlo dropout (MCDropout)---aim to improve the interpretability of machine learning models by assigning an estimated uncertainty value to each of their prediction outputs. 
However, since too \textit{high} uncertainty estimates can have fatal consequences in practice, this paper analyzes the above techniques.
Firstly, we show that GP methods always yield \textit{high} uncertainty estimates on out of distribution (OOD) data.
Secondly, we show on a 2D toy example that both BNNs and MCDropout do \textit{not} give high uncertainty estimates on OOD samples.
Finally, we show empirically that this pitfall of BNNs and MCDropout holds on real world datasets as well.
Our insights 
(i)~raise awareness for the more cautious use of currently popular UE methods in Deep Learning, 
(ii) encourage the development of UE methods that approximate GP-based methods---instead of BNNs and MCDropout, and
(iii) our empirical setups can be used for verifying the OOD performances of any other UE method. The source code is available at \href{https://github.com/epfml/uncertainity-estimation}{\url{https://github.com/epfml/uncertainity-estimation}}.
\end{abstract}

\section{Introduction}

\freefootnote{$^\star$Equal contribution. \\ \hspace*{2.05em} Contact:
\href{mailto:yehao.liu@epfl.ch}{$\{$yehao.liu},
\href{mailto:matteo.pagliardini@epfl.ch}{matteo.pagliardini$\}$@epfl.ch},
\href{mailto:tatjana.chavdarova@berkeley.edu}{tatjana.chavdarova@berkeley.edu},
\href{mailto:stich@cispa.de}{stich@cispa.de}.
}

While the complexity of deep learning models allows them to outperform humans on a growing number of tasks, it often renders them \textit{uninterpretable} to human users.  This limits the use of these methods in applications where decisions can have significant consequences, such as for instance in the fields of health care, finance or autonomous driving.
Uncertainty estimation (UE) methods aim to improve model interpretability, by associating an estimate of its uncertainty to each output~\citep{kim2016}.
As such, uncertainty estimates are expected to be high in value whenever we give as input a sample that is out of the distribution (OOD) of the dataset that the model was trained with.

Gaussian processes (GP) are considered as the gold standard for UE~\citep{rasmussen2005gp,van2021feature}.
Similarly, Bayesian Neural Networks (BNNs) yield  mathematically grounded UE methods that extends standard neural network (NN) training to posterior inference.
However GPs and BNNs do not scale well with the dimension of the input data and the parameters space, respectively, and are often infeasible for real-world problems.
Thus, one of the most popular methods in deep learning is Monte Carlo Dropout (MCDropout), which method applies Dropout~\citep{srivastava2014dropout} at inference time to compute uncertainty estimates.
An active branch of work is focused on finding better and more efficient approximation methods for UE. 
In order for the community to move forward in the right direction, it is important to verify that the these UE methods work well and to understand their differences and  flaws. \looseness=-1

\noindent\textbf{Overview of contributions.} 
We (i) show on a simplistic 2D example that 
surprisingly \textit{both} BNNs and MCDropout methods might estimate low uncertainty  for OOD data, while GP can detect OOD samples.
We (ii) we empirically demonstrate that the above perils of BNNs and MCDropout  occur in real-world situations such as on MNIST---for OOD data generated by interpolation between two real samples---and ResNet-18 trained on CIFAR-10---for samples from CIFAR-100 that were not used during training.
Finally, we (iii) argue analytically on a simple example why GP methods succeed.

\section{Background and Related Works}\label{sec:background}
Gaussian processes~\citep[GPs, ][]{rasmussen2005gp} 
are a non-parametric distance aware output function.
GPs model the similarity between data points
with a kernel function, and use the Bayes rule to model a \textit{distribution over functions} by maximizing the marginal likelihood~\citep{rasmussen2005gp}. %
As such, GPs require access to the full datasest at inference time, and although there exsit approximations, this family of methods does \textit{not} scale well with the dimension of the data.   

Bayesian Neural Networks (BNNs) are stochastic neural networks trained using a
Bayesian approach. While standard neural network training performs a maximum likelihood estimation (MLE) of the parameters of the network $\bm{\omega}\! \in \!\Omega$, %
training BNNs extends to %
estimating the posterior distribution.
While the  formulation is straightforward (see App.~\ref{app:uncertainty}), due to the integration with respect to the whole parameter space $\Omega$ this computation is intractable for DNNs. 
In this work, we train the BNNs with two common training methods: \ 
Mean-Field Variational Inference (MFVI)~\cite{blei2017}, and Hamiltonian Monte Carlo (HMC)~\cite{neal2012}---the latter expected to be more accurate. %

Lastly, Monte Carlo Dropout (MCDropout) applies Dropout at inference time allows for computationally efficient UE.
\cite{gal2016dropout}  
show that \textit{Dropout}~\citep{srivastava2014dropout}
approximates Bayesian inference of a GP. 
Given $M$ models $\{\mathcal{C}_{\vomega}^{m}\}_{m=1}^M$ sampled with MC Dropout, where each outputs non-scaled values---``logits'', we define %
$\hat{\vy} \in \R^C$ as the average prediction:
$
    \hat{\vy} \! = \! \frac{1}{M} \smash{\sum_{m=1}^M} \mathop{\rm Softmax}\big(\mathcal{C}^{m}_\vomega(\vx)\big).
$
Given $\hat{\vy}$, there are several ways to estimate the model's uncertainty, see~\citep[][ \S 3.3]{Gal2016Uncertainty}.
We use the \textit{entropy} of the output distribution %
(over the classes) to quantify the uncertainty estimate of a given sample $\vx$:
$
    \mathcal{H}(\vx, \vomega)\! =\! - \sum_{c \in C}  \hat{\vy}_c \log \hat{\vy}_c \,.
$

\noindent\textbf{Relevant Works}.
Closest to ours  are \citep{amersfoort2020} and the recent~\citep{van2021feature}  where the authors point out on a 2D toy example that the uncertainty estimates of Deep Ensembles~\citep{deepensemb17} and Deep Kernel Learning~\citep{wilson2015deep} models, respectively, are  \textit{low} on points that are ``far away'' from the regions where the training data are located.
\citet{foong2019} point out limitations in the expressiveness of the predictive uncertainty estimate given by mean-field variational inference, in particular that the method fails to give calibrated uncertainty estimates in between separated regions of observations.  \looseness=-1 %

\begin{figure*}[tb]
    \centering
    \hfill
    \begin{subfigure}[b]{0.222\textwidth}
        \centering
        \includegraphics[width=\textwidth]{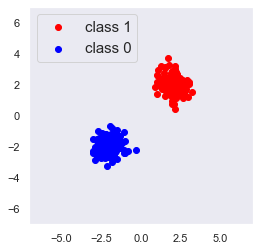}
        \caption[Location of data points]%
        {Dataset}  %
        \label{fig:illustration-data}
    \end{subfigure}
    \hfill
    \begin{subfigure}[b]{0.22\textwidth}   
        \centering 
        \includegraphics[width=\textwidth,trim={0 0 1.7cm 0},clip]{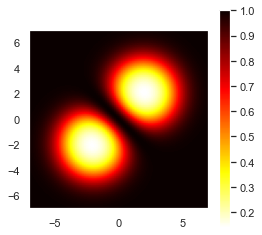}
        \caption[Value of GP predict uncertainty. The lighter the lower.]%
        {Gaussian Process}    
        \label{fig:illustration-gp}
    \end{subfigure}
    \hfill
    \begin{subfigure}[b]{0.22\textwidth}   
        \centering 
        \includegraphics[width=\textwidth,trim={0 0 1.7cm 0},clip]{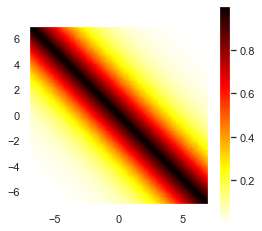}
        \caption[]%
        {MCDropout}    
        \label{fig:illustration-MCD}
    \end{subfigure}
    \hfill
    \begin{subfigure}[b]{0.22\textwidth}   
        \centering 
        \includegraphics[width=\textwidth,trim={0 0 1.7cm 0},clip]{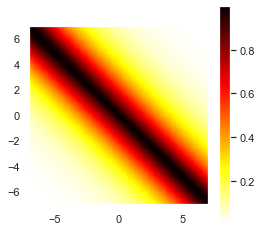}
        \caption[]%
        {BNN}    
        \label{fig:illustration-BNN}
    \end{subfigure}
    \hfill\null
    \caption{
    Uncertainty estimates by GP and MCDropout on a 2D dataset, depicted in (\subref{fig:illustration-data}). %
    We display the uncertainty estimates of GP (\subref{fig:illustration-gp}), MCDropout (\subref{fig:illustration-MCD}), and BNN trained with HMC (\subref{fig:illustration-BNN}), for points of the 2D input space (the darker, the higher the uncertainty).
    See \S~\ref{sec:num_exp} for discussion and  App.~\ref{app:implementation} for implementation details.
    } 
    \label{fig:illustration}
\end{figure*}

\section{The Failure Mode of Popular Uncertainty Estimation Methods} %

We now first illustrate the behaviour of GP and MCDropout on a set of examples---see details on the experimental setups in App.~\ref{app:implementation}, and then study the GP behaviour theoretically.

\subsection{Numerical Illustrations}\label{sec:num_exp}%

\noindent\textbf{Motivating Two-Dimensional Example (Figure~\ref{fig:illustration}).}
We consider a binary classification  problem with two-dimensional bi-modal input data, generated from $\mathcal{N}(-2,-2)$ and $\mathcal{N}(2,2)$ with label $1$ and $0$, resp. 
A Bayes optimal classifier is given by $\mathop{\rm sign}(x-y)$ for ($x,y)$-coordinates.
We train a GP model and a two-layer perceptron on the training data. Both resulting models can perfectly classify the training data. 
We observe in Fig.~\ref{fig:illustration}  that GP gives high uncertainty estimates for input data far away from the two modes, while BNNs and MCDropout both assign high uncertainty only near to the decision boundary ($x\!=\!-y$) and fail to identify many OOD regions of the input space.

\begin{figure*}[!tb]
    \centering
    \begin{subfigure}[b]{1\textwidth}
        \centering
        \includegraphics[width=0.9\textwidth,trim={0 0.2cm 0 0},clip]{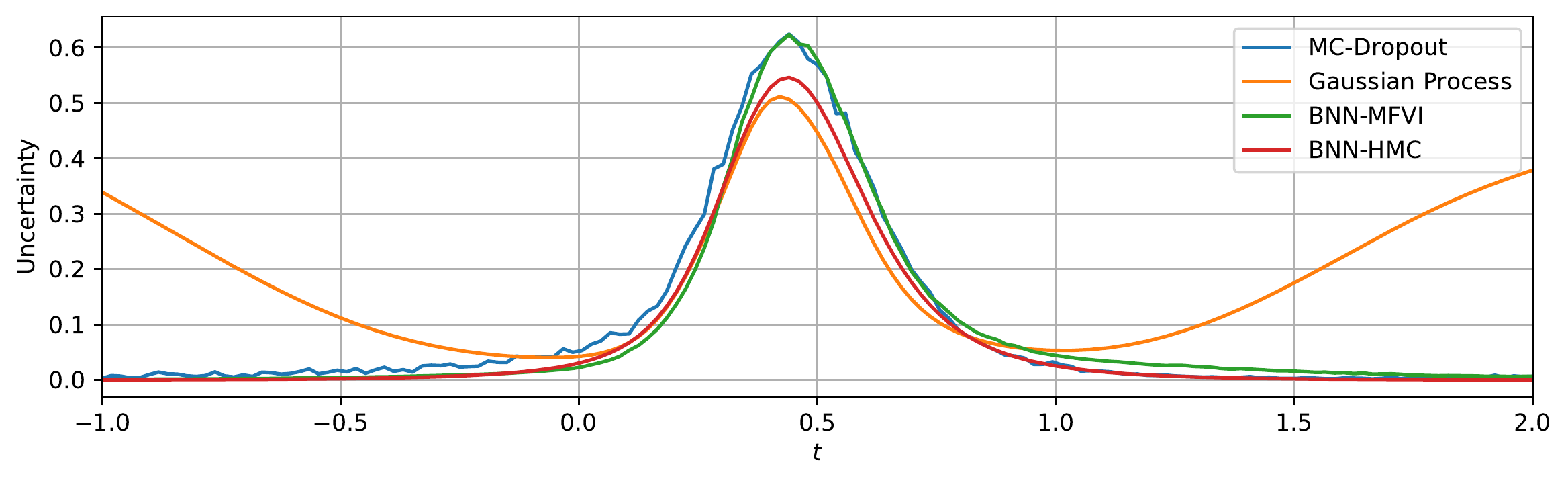}
    \end{subfigure}
    \begin{subfigure}[b]{1\textwidth}   
        \centering 
        \includegraphics[width=0.9\textwidth,trim={-1cm 0 0cm 0.2cm},clip]{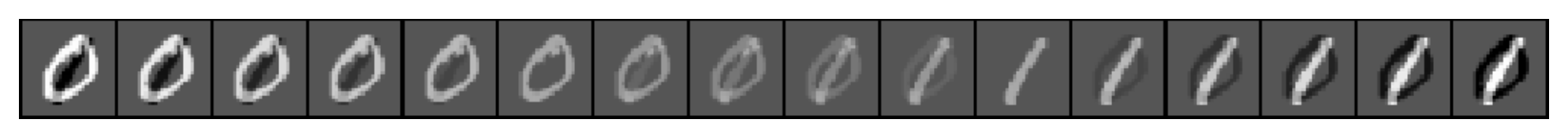}
    \end{subfigure}
    \caption[]
    {
    Estimated uncertainty by GP, BNN and MCDropout. We train a GP model, a NN classifier with dropout, and two BNNs on MNIST handwritten digits of {\tt 0} and {\tt 1}'s and display the uncertainty estimation on linear interpolated images of the form 
    $t \cdot \vx^{1} + (1-t) \cdot \vx^{0}$ 
    for $t \in [-1,2]$ (grayscale values adjusted for better display in the illustration). The BNNs were trained with two different training methods: Mean-Field Variational Inference (MFVI), and Hamiltonian Monte Carlo (HMC). 
     All methods give high uncertainty for convex combinations of training data ($t \in [0,1]$, highest uncertainty for $t \approx 0.5$).
     GP detects OOD data for $t \not\in [0,1]$ while BNN/MCDropout estimates both have low uncertainty. This demonstrates that BNN and MCDropout fail to detect this kind of OOD data. 
    }
    \label{fig:mnist}
\end{figure*}

\begin{figure}[!tbp]
  \centering
  \begin{subfigure}[t]{.28\linewidth}
  \includegraphics[width=\textwidth,clip]{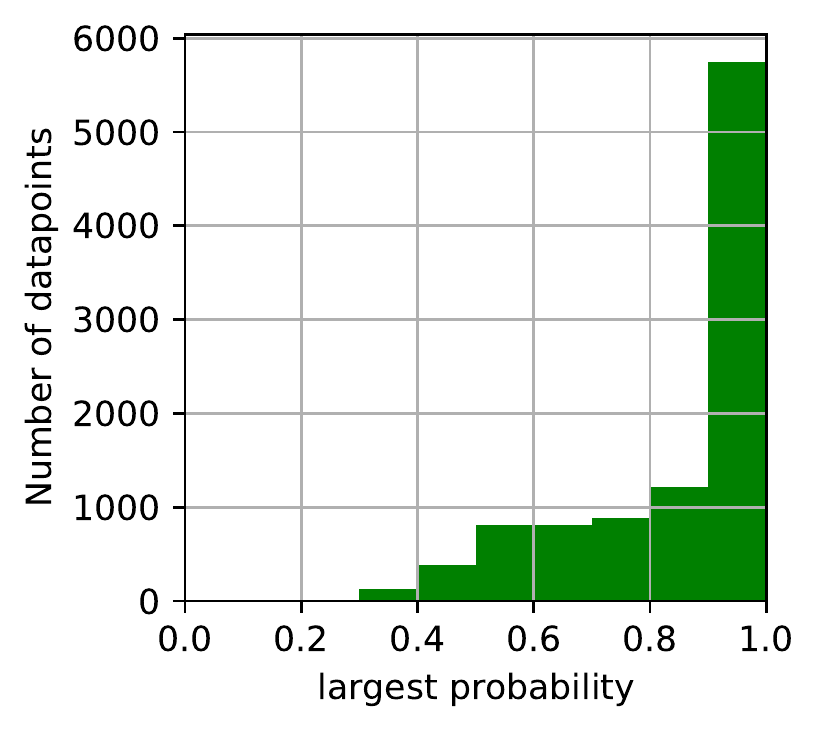}
  \caption{Without MCDropout}\label{subfig:hist-probas-cifar100-nodropout}
  \end{subfigure}
  \begin{subfigure}[t]{.28\linewidth}
  \includegraphics[width=\textwidth,clip]{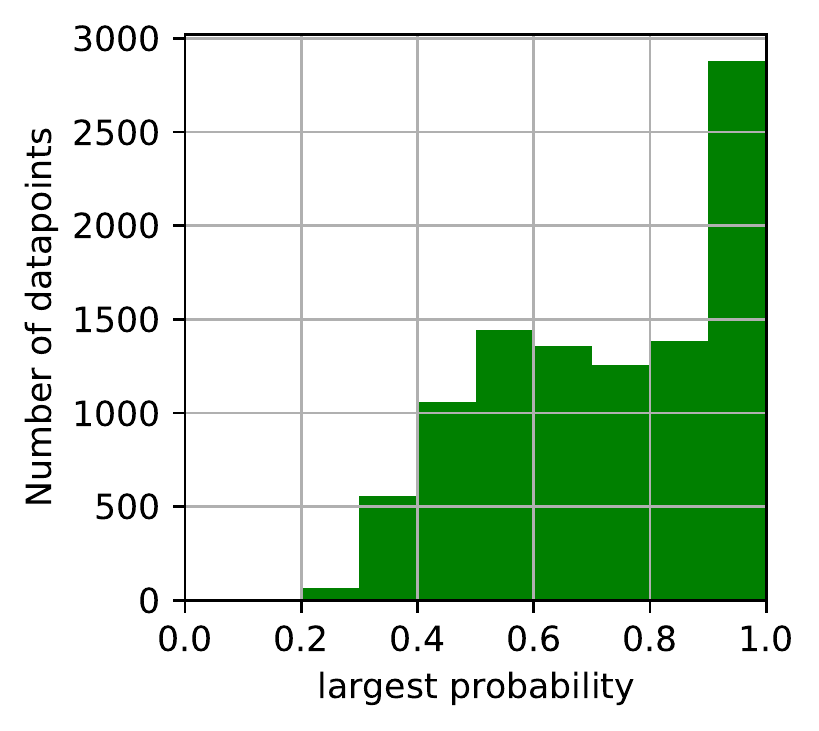}
  \caption{With MCDropout}\label{subfig:hist-probas-cifar100}
  \end{subfigure}
  \begin{subfigure}[t]{.38\linewidth}
  \includegraphics[width=\textwidth,clip]{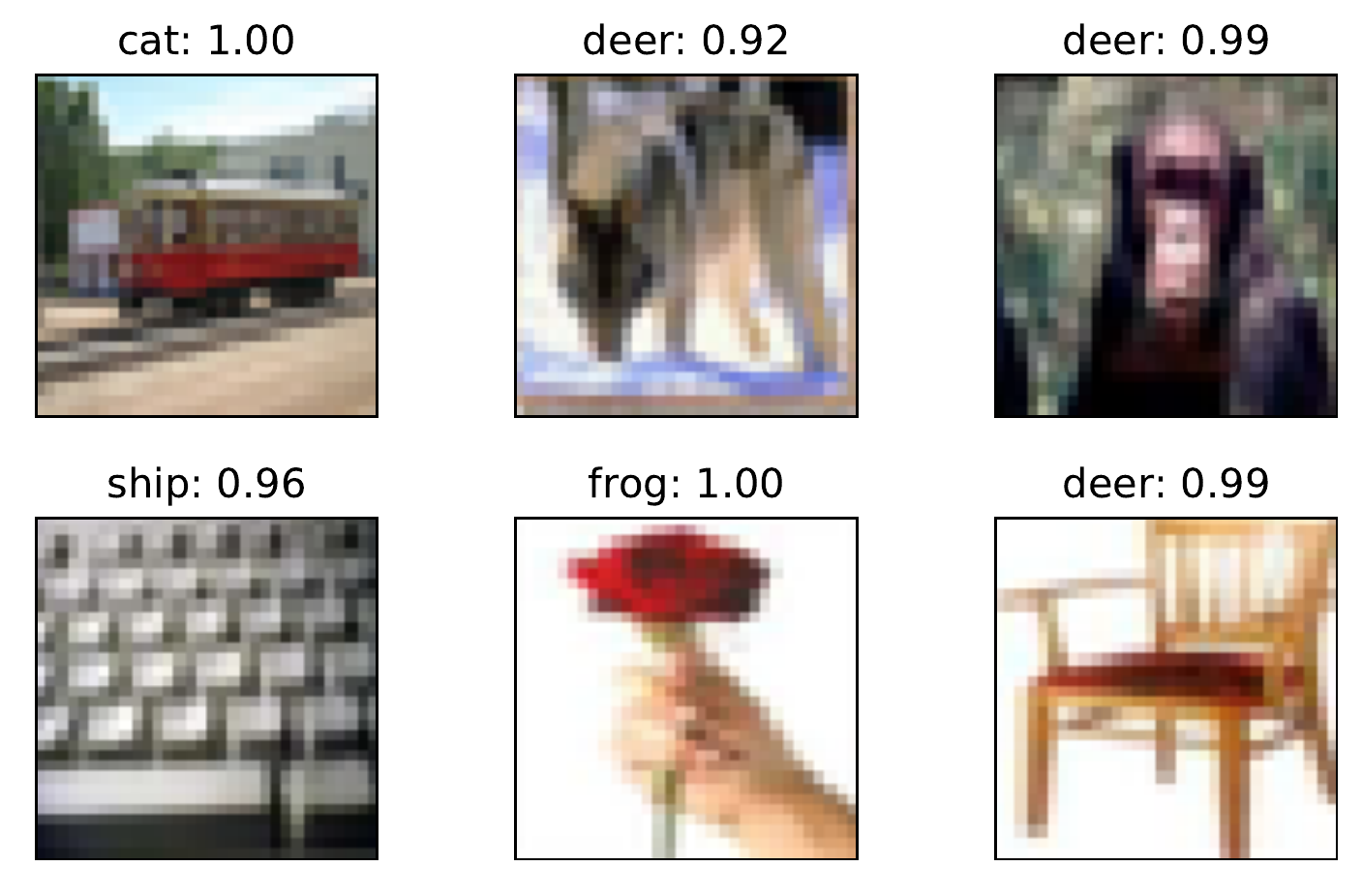}
  \caption{CIFAR100 predictions and confidence
  }\label{subfig:examples-high-proba-cifar100}
  \end{subfigure}
 \caption{MCDropout struggles detecting OOD samples:\ we trained a ResNet-18 on CIFAR-10 and used it to classify images from CIFAR-100. We observe in (\subref{subfig:hist-probas-cifar100}) that many CIFAR-100 datapoints, which are outside of the training distribution, are nonetheless inducing a confident prediction, despite using MCDropout. We show in (\subref{subfig:examples-high-proba-cifar100}) some CIFAR-100 images along with their class and confidence, obtained sampling 20 models via MCDropout. 
 }\label{fig:ood-cifar100}
\end{figure}

\noindent\textbf{MNIST data (Figure~\ref{fig:mnist}).}
We train GP and MCDropout on the digits of class {\tt 0} and {\tt 1} of the MNIST~\citep{lecun2010mnist} dataset. We first randomly select two data points of each class, $\vx^{0}$, $\vx^{1}$. We then test the models on artificial data generated by linear combination of the form $t \cdot \vx^{1} + (1-t) \cdot \vx^{0},  t \in [-1,2]$. Although these samples are not necessarily realistically looking images, we expect that the uncertainty estimates are high when $t$ is close to $0.5$, \textit{and} for $t \not\in [0,1]$. However, in Fig.~\ref{fig:mnist} we see that GP successfully detects OOD data for $t \not\in [0,1]$ while MCDropout fails.
However, we observed that MCDropout is able to detect other kinds of OOD data, such as e.g.\ the MNIST digits {\tt 2}--{\tt 9} that where not used for training the model (See Tab.~\ref{tab:mnist} in App.~\ref{app:additional}.)
To follow this experimental protocol for an arbitrary method, we provide some helper code, see App.~\ref{app:code_instructions} for details.

\noindent\textbf{CIFAR data (Figure~\ref{fig:ood-cifar100}).}
Lastly, we train ResNet-18~\citep{he2016deep} with dropout on \mbox{CIFAR-10}~\citep{Krizhevsky2012:cifar10} and measure the uncertainty with MCDropout on samples from \mbox{CIFAR-100}. In Fig.~\ref{fig:ood-cifar100} we observe that although MCDropout improves over model without MCDropout enabled, it yet fails to detect these OOD samples.

\subsection{GP Analysis on OOD samples}
The following result shows that for a test data sample that is not correlated with the training data---measured with the kernel---the estimated uncertainty is \textit{high}, the proof is omitted and given in App.~\ref{app:proof}. 

\begin{thminformal}\label{thm:1}
Consider a n-dimensional binary GP classification model with kernel $K(\xx,\xx')$.
Let $\mathbf{k}_\star$ denote the similarity vector $(\mathbf{k}_\star)_i= K(\xx_i,\xx_\star)$ between a test sample $\xx_\star$ and training data $\xx_i$, $i=1,\dots, n$. If $\|\mathbf{k}_\star\|$ is small---i.e.\ if the similarity between the test sample and the  training data points is low---then the GP prediction for $\xx_\star$ will have high uncertainity.
\end{thminformal}
This theorem just argues one direction and in particular does not show that the uncertainity is low only for in-distribution data. Moreover, if the kernel overestimates the similarity of the data, OOD samples cannot be detected. 
We leave the extension of this statement to higher dimensions and more general models to future work.

\section{Discussion \& Conclusion}
We found that the GP model performed well in OOD detection in all the tasks and experiments we studied. We analyze GPs analytically and argue that access to a kernel function that measures similarity to the training data can explain this behavior. In contrast, BNNs and standard NN might only memorize training data implicitly and struggle to detect some OOD data points. 

While an active line of work focuses on finding alternative approximations of BNNs~\citep{ober2021global,mbuvha2021separable}, in this work we point out that BNNs and MCDropout are unable to detect certain OOD samples. Our results advocate the need for efficient UE techniques that approximate GP-based methods and that MCDropout's predictions should be taken with a grain of salt---although this is the most efficient and frequently used method in practice.
We believe that our contribution will lead the community to collaborate on new benchmarks and tests of UE methods, in order to develop methods with theoretically guaranteed properties.  This is necessary to develop future-proof UE technologies that can be used for safety-critical applications in the years to come.

\section*{Acknowledgements}
The authors would like to thank Prof. M. Jaggi (EPFL) for his support. TC is supported by the Swiss National Science Foundation (SNSF), grant P2ELP2\_199740. 

{\small
\bibliography{main}
\bibliographystyle{abbrvnat}
}

\clearpage
\appendix
\section{Overview of Uncertainty Estimation Methods}\label{app:uncertainty}

Below, we first describe in more detail BNNs---see also~\citep{kendall2017} for a full review of UE methods, and then list relevant works to ours.

\subsection{Bayesian Neural Networks}\label{app:bnn}

\paragraph{Bayesian Neural Networks (BNNs).}
 Given $N$ training datapoints $\mathcal{D} = \{(\bm{x}_i, y_i)\}_{i=1}^N$, 
training BNNs extends to posterior inference by estimating the posterior distribution:
\begin{equation}
    p(\bm{\omega} | \mathcal{D}) = \frac{p(\mathcal{D}|\bm{\omega}) p(\bm{\omega})}{p(\mathcal{D})} \,,
\end{equation}
where $p(\bm{\omega})$ denotes the prior distribution on a parameter vector $\bm{\omega} \in \Omega$.
Given a new  sample $\bm{x}^\star,y^\star$, the predictive distribution is then:
\begin{equation}
    p(y^\star|\bm{x}^\star, \mathcal{D}) =  \int_{\Omega} p(y^\star|\bm{x}^\star, \bm{\omega}) p(\bm{\omega}|\mathcal{D}) d\bm{\omega}\,.
\end{equation}
In \S~\ref{sec:background} we point out that BNNs are often computationally expensive, what arises due to the above integration with respect to the whole parameter space $\Omega$.

\paragraph{Training BNNs.}
Two methods exist to train BNNs:
\begin{itemize}[leftmargin=2em]
    \item \textit{Markov Chain Monte Carlo (MCMC) algorithm}---samples the posterior directly, however requires to cache a collection of samples $\{\vomega_k\}_{k=1}^K$.
    \item \textit{Variational inference approach (``Variational Bayes'')}---learns a variational distribution $q(\vomega)$ to approximate the exact posterior.
\end{itemize}

Regarding the former, one way to use posterior uncertainty is to sample a set of values $\vomega_1, \dots, \vomega_K$ from a posterior $p(\vomega|\mathcal{D})$, and then average their predictive distributions:
$$
p(y^\star| \vx^\star,\mathcal{D} ) \approx  \sum_{k=1}^{K} p(y^\star | \vx^\star, \vomega_k) \,.
$$
In the context of BNNs most popular choice to replace the sampling from the posterior is  \textit{Hamiltonian} MC (HMC)~\citep[see e.g.,][]{neal2012,betancourt2013hamiltonian}, also known as \textit{hybrid} Monte Carlo, which method uses the derivatives of the density function being sampled to generate efficient transitions spanning the posterior.
HMC uses an approximate Hamiltonian dynamics simulation based on numerical integration which is then corrected by performing a Metropolis acceptance step.
Unfortunately, HMC does  not scale to large datasets, because it is inherently a batch algorithm---it requires visiting the entire training set for every update.

As the former MCMC approach requires revisiting each data point for each update, the latter approximate approach of Variational inference scales better, and thus this approach gained a lot of popularity in the context of BNNs.
Variational Bayes approximates the complicated posterior distribution with a ``simpler'' variational approximation $q(\vomega)$, for example a Gaussian posterior with a diagonal covariance, (i.e.\ fully factorized Gaussian), and in that case each parameter of the model has its own mean and variance.
Analogously to variational autoencoders~\citep[VAEs, ][]{kingma2014autoencoding}, we define a variational lower bound:
$$
\log p(\mathcal{D}) \geq \mathcal{F}(q)  \triangleq \underbrace{ \E_{q(\vomega)} [ \log p(\mathcal{D} | \vomega) ]}_{\text{Likelihood term}} - \underbrace{\mathds{D}_{KL} \big(q(\vomega) || p(\vomega)\big)}_{\text{KL term}} \,,
$$
where $\mathds{D}_{KL}$ denotes the Kullback–Leibler (KL) divergence, and the KL term encourages $q$ to match the prior.
Unlike VAEs, $p(\mathcal{D})$ is fixed, and we are \emph{only} maximizing $\mathcal{F}(q)$ with respect to the variational posterior $q$ (\textit{i.e.} a mean and standard deviation for each weight).
Same as for VAEs, the gap equals the KL divergence from the true posterior:
$
\mathcal{F} (q) = \log p(\mathcal{D}) - \mathds{D}_{KL} \big(q(\vomega) || p (\vomega|\mathcal{D})\big) \,.
$
Hence, maximizing $\mathcal{F}(q)$ is equivalent to approximating the posterior.

See for example~\citep{jospin2021handson} for more detailed discussion on the differences between these two BNN approaches.

\subsection{Additional Related Works}\label{app:related_works}

\cite{amersfoort2020} combine Deep Kernel Learning~\citep{wilson2015deep} framework and GPs by using NNs to learn low dimensional representation where the GP model is jointly trained, resulting in a method called \textit{Deterministic Uncertainty Quantification} (DUQ).
More recently, \cite{van2021feature} showed that the DUQ method based on DKL \textit{can} map OOD data close to training data samples, referred as ``feature collapse''.
The authors thus propose the \textit{Deterministic Uncertainty Estimation} (DUE) which in addition to DUQ method ensures that the encoder mapping is bi-Lipschitz.
It would be interesting to explore if these methods perform well on similar OOD experiments as those considered in this work.

Other pitfalls of MCdroput have been pointed out in~\citep{osband2016} where it is argued that suggest MCdropout approximates the risk, not to uncertainty.

In this work we focused on the most popular UE methods in machine learning, however, it is worth noting that other approaches exist. For example, the \textit{jackknife$+$} method~\citep{barber2020predictive} is known to have good coverage guarantees, however it is less popular in deep learning where datasets are typically large, due to its requirement to retrain a model on multiple subsets of the full training set. We leave analyzing the OOD performance of other UE methods for future work.

\section{Details on the Implementation}\label{app:implementation}

\subsection{Experiments on 2D-Toy data}\label{app:impl_toy}
In this section we list the details of the implementation of Fig.~\ref{fig:illustration}. The 2D toy dataset is generated sampling two normal distributions. For class $1$, we sample 200 points from $\mathcal{N}(\boldsymbol{\mu}_1,\sigma)$, with $\boldsymbol{\mu}_1 \!= \!(2,2)$ and $\sigma\!=\! \text{diag}((0,1, 0.1))$. For class $0$, we sample 200 points from $\mathcal{N}(\boldsymbol{\mu}_0,\sigma)$, with $\boldsymbol{\mu}_0 \!=\! (-2,-2)$.  

To train our Gaussian Process model,  we use the \texttt{GaussianProcessClassifier} function\footnote{\url{https://scikit-learn.org/stable/modules/generated/sklearn.gaussian_process.GaussianProcessClassifier.html}} from the sklearn library \cite{pedregosa2011scikit}, using an RBF kernel. The uncertainty is then calculated as the entropy of the predicted probability distribution, see \S~\ref{sec:background}.

For our MCDropout model, we use an MLP with one hidden layer that contain $300$ neurons. The activation function is ReLU and the dropout rate is $0.5$. The uncertainty is calculated via average entropy of $100$ forward dropout. The network is optimised by Adam with a cross-entropy loss using the PyTorch library~\citep{paszke2019pytorch}. 

We tuned hyperparameters extensively. Using grid search, we tried different combinations of dropout rates and regularisation coefficients.
(Additionally, we also explored increasing the number of neurons up to $10,000$, increasing the number of hidden layers up to $4$, changing the activation function, trying different values of dropout and regularization. However, none of these variations led to a change of our main observations.) 

For BNNs, we use the approximation from ``Bayes by Backprop''~\citep{blundell2015weight}, implemented through the blitz Python package~\citep{esposito2020blitzbdl}. The network has two hidden layers, with $512$ neurons for the first layer, and $128$ neurons for the second layer. We use a regularization coefficient of $0.1$.

\subsection[Experiments on MNIST]{Experiments on MNIST}
\label{app:mnist_exp_app}

We select the {\tt 0} and {\tt 1} digits from the MNIST dataset~\citep{lecun2010mnist}, and train the different models using solely these two classes.
To train a GP on MNIST, we first train a Multi-layer Perceptron (MLP) of three layers---each of 600, 20 and 2 units, resp., using a dropout rate $0.6$. 
We then use the first two layers as encoder $\mathcal{E} \colon \vx\mapsto\R^{20}$, $\vx \in \R^{28{\times}28}$, which is kept fixed during the training of the GP model.
Thus the GP model is trained given the embeddings of the input images, and the GP implementation is as in~\ref{app:impl_toy}.

For MCDropout, we used smaller network of two layers of $500$ and $2$ units each, and after sampling $M\!=\!100$ models we use the entropy of the output distribution to compute the uncertainty estimates.

Finally, the architecture used for the BNN uncertainty estimation method is an MLP with three hidden layers of $1024,128$ and $2$ units each. We used a regularization coefficient of 0.1.

In Fig.~\ref{fig:mnist} we designed the experiment as follows:\ to measure the uncertainty, we randomly select two data points of each class, $\vx^{0}$, $\vx^{1}$. We then test the models on artificial data generated by linear combination of the form $t \cdot \vx^{1} + (1-t) \cdot \vx^{0},  t \in [-1,2]$. We also explore the result for $t$ greater than $1$ and less than $0$ although it might not make sense in the real world. Then we calculate the uncertainty for different value of $t \in [-1,2]$. We display the values obtained from repeating this procedure 100 times and averaging the values.

For the implementation details, the package and the function are the same as previous section. And all these methods achieve accuracy more than $99.9\%$ on the test dataset.

The Gaussian Process is Deep GP. We at first train a MLP, with two hidden layers and dropout rate 0.6. The first hidden layer has 600 neurons, while the second one only has 20. This stage is feature extraction.  For each data, the input space is mapped from $28 \times 28$ to 20. After this network is well-trained, we use the value of these 20 new features as inputs to train Gaussian Process. 

The implementation of MCDropout is more direct. The network has 500 neurons and dropout rate 0.6. The uncertainty is calculated via average entropy of 100 forward dropout. Again, we tried different structure of the network and it doesn't improve at all.

For BNNs. The networks used have two hidden layers, with $1024$ neurons in the first one and $128$ neurons in the second one. Regularization coefficient is $0.1$. The BNN-HMC implementation is done using the \texttt{hamiltorch}\footnote{\texttt{https://github.com/AdamCobb/hamiltorch}} library, with a prior precision of $5$ for each parameter of the model , a trajectory length of $3$, and a step size of $0.0005$.

\subsection[Experiments on CIFAR-10]{Experiments on CIFAR-10 \citep{Krizhevsky2012:cifar10}}
We use a ResNet-18 architecture \citep{resnet} modified to accommodate the MCDropout sampling procedure. The modification consists of adding a dropout layer with dropout probability $p=0.2$ after each convolutional layer. We train with an SGD optimizer with a momentum of $0.9$ and a weight decay of $0.00001$. The learning rate is following a triangular scheduler, first increasing linearly from $0$ to $0.2$ during the first half of the epochs, and then decreasing linearly back to $0$. We train for a total of $30$ epochs, reaching a test accuracy of $89\%$. When using MCDropout, we sample $20$ models and average the output distributions. The images of Fig. \ref{subfig:examples-high-proba-cifar100} are taken randomly from the CIFAR100 test images with an MCDropout predicted probability larger than $0.9$. 

\subsection{How to use the provided source code for OOD evaluation}\label{app:code_instructions}

We provide some code to help reproduce this experimental protocol, see: \href{https://github.com/epfml/uncertainity-estimation}{\url{https://github.com/epfml/uncertainity-estimation}}. The code in the notebook \texttt{0-1-interpolation-for-UE-evaluation.ipynb} shows how to build the test digits and how to compute the uncertainty for a simple classifier.

\section{Additional Numerical Results}
\label{app:additional}

\begin{table}[H]
    \centering
    \begin{tabular}{ccccc}
    \toprule
        Class (digit) & Uncertainty & \phantom{some space} & Class (digit) & Uncertainty\\  \midrule
         \tt 0 & \textbf{ 0.0370} & & \tt 5 & 0.4963\\ 
        \tt1 & \textbf{0.0200} & &  \tt6 & 0.4861\\
        \tt2 & 0.3304 & &  \tt 7 & 0.4949\\ 
        \tt3 & 0.4344 & & \tt 8 & 0.4775\\ 
        \tt4 & 0.4689 & & \tt 9 & 0.4832\\ \bottomrule
    \end{tabular}\vspace{2mm}
    \caption{The uncertainty prediction of MCDropout for the digit {\tt 0} to digit {\tt 9}. With training set only {\tt 0} and {\tt 1}, and measured on all of the test set. The number is the average uncertainty for all of these numbers. We saw that number {\tt 2}--{\tt 9} have high uncertainty, while the uncertainty for number {\tt 0} and {\tt 1} are small. }
    \label{tab:mnist}
\end{table}

\section{Proof of Theorem~\ref{thm:1}}
\label{app:proof}

Our notation follows that of~\citep[pages 41--43]{rasmussen2005gp}, and it is summarized in Table~\ref{tab:notation}. %

\begin{table}[h]
    \centering
    \begin{tabular}{cl} 
        \toprule
        $k(\boldsymbol{x},\boldsymbol{x'})$  & Kernel value between $\boldsymbol{x}$ and $\boldsymbol{x'}$ \\
        $X$ & Training input data\\
        
        $\vy$ & Ground-truth labels\\
        $\vx_\star$ & Test data point \\
        $K$ & covariance matrix for the (noise free) $\mathbf{f}$ values. i.e $ k(X,X)$\\
        $\mathbf{k_\star}$ & Kernel matrix between test and training input. i.e $k(\mathbf{x_\star},X)$\\
        $ \bar{\pi}_{\star}$ & The predict probability of new test sample $\xx_\star$ \\
        $\mathbf{f}$ & Latent function values \\
        $\hat{\mathbf{f}}$ & Calculated by  $\mathbb{E}_{q}[\mathbf{f} \mid X, \mathbf{y}]$, the maximum posterior.\\ 
        $f_\star$ & Gaussian process (posterior) prediction (random variable)\\
        $W$ & The negative Hessian of the $\log p(\mathbf{y} \mid \hat{\mathbf{f}})$.\\   \bottomrule
    \end{tabular}\vspace{2mm}
    \caption{Summary of the notation.}
    \label{tab:notation}
\end{table}

\subsection{Assumption}
We study a Gaussian Process $f$ with n-D training data $X$ with  kernel $K$. We assume the binary classification setting (Gaussian Process Classification with Laplace approximation~\citep{Williams1998}).

\subsection{Predictive probability}
The predict probability of new test sample $\xx_\star$ is approximated via
\begin{equation}
    \bar{\pi}_{\star} \simeq \int \sigma\left(f_{\star}\right) q\left(f_{\star} \mid X, \mathbf{y}, \vx_{\star}\right) d f_{\star} \,,
\end{equation}
where $\sigma(z)$ is a link function (logistic or Gaussian CDF),          $q\left(f_{\star} \mid X, \mathbf{y}, \vx_{\star}\right)$ is a Gaussian distribution whose mean and variance are given below

\subsection{Parameters for Gaussian distribution above}
The parameters for $q\left(f_{\star} \mid X, \mathbf{y}, \mathbf{x}_{\star}\right)$ are

Mean: $$\mathbb{E}_{q}\left[f_{\star} \mid X, \mathbf{y}, \mathbf{x}_{\star}\right]= \mathbf{k}_\star^{\top} K^{-1} \hat{\mathbf{f}}=\mathbf{k}_\star^{\top} \nabla \log p(\mathbf{y} \mid \hat{\mathbf{f}})\,.$$

Variance: $$\begin{aligned} \mathbb{V}_{q}\left[f_{\star} \mid X, \mathbf{y}, \mathbf{x}_{\star}\right] &=k\left(\mathbf{x}_{\star}, \mathbf{x}_{\star}\right)-\mathbf{k}_{\star}^{\top} K^{-1} \mathbf{k}_{\star}+\mathbf{k}_{\star}^{\top} K^{-1}\left(K^{-1}+W\right)^{-1} K^{-1} \mathbf{k}_{\star} \\ &=k\left(\mathbf{x}_{\star}, \mathbf{x}_{\star}\right)-\mathbf{k}_{\star}^{\top}\left(K+W^{-1}\right)^{-1} \mathbf{k}_{\star} \,, \end{aligned}$$
where  $\hat{\mathbf{f}}=\mathbb{E}_{q}[\mathbf{f} \mid X, \mathbf{y}]$.

\subsection{Theorem}
\begin{theorem}
Consider the $n$-dimensional binary GP classification model described above with kernel $K(\xx,\xx')$.
Let $\mathbf{k}_\star$ denote the similarity vector $(\mathbf{k}_\star)_i= K(\xx_i,\xx_\star)$ between a test sample $\xx_\star$ and training data $\xx_i$, $i=1,\dots, n$. If $\|\mathbf{k}_\star\| \leq \epsilon$ is small
then Gaussian Process Classification with Laplace approximation will assign a probability close to 0.5 for both of two classes, with probability approaching $\frac{1}{2}$ as $\epsilon \to 0.$
\end{theorem}

\begin{proof}
As can be seen from the derivations above, if  $\|\mathbf{k}_\star\| \leq \epsilon$ is small, then the mean tends to zero as $\epsilon \to 0$, and variance tends to $k(\xx_{\star},\xx_{\star})$. We now  calculate the predictive probability with this mean and variance. 

\begin{equation}
    \begin{split}
\bar{\pi}_{\star}:&= \int \sigma(z) \mathcal{N}\left(z \mid \bar{f}_{\star}, \mathbb{V}\left[f_{\star}\right]\right) d z \\
&= \int \sigma(z) \mathcal{N}\left(z \mid 0, k(\xx_\star,\xx_\star) \right) d z \\
&= \int (\sigma(z)-\frac{1}{2}) \mathcal{N}\left(z \mid 0, k(\xx_\star,\xx_\star) \right) d z + \frac{1}{2} \int \mathcal{N}\left(z \mid 0, k(\xx_\star,\xx_\star) \right) d z  \\
& \stackrel{\epsilon \to 0}{=} \frac{1}{2} \,.
    \end{split}
\end{equation}

In the third equation, the first term vanish because it is an odd function. Hence, we showed that when $\|\mathbf{k}_\star\| \leq \epsilon$ is small, the prediction tends to $\frac{1}{2}$ for $\epsilon \to 0$.
\end{proof}
As we know, the entropy is the largest when two classes both have probability 0.5 in binary classification. This means that if the testing input is far away from training input, then GP would return a high uncertainty for this input.

\end{document}